\documentclass[11pt]{article}
\usepackage{natbib}
\usepackage{algorithm}
\usepackage{verbatim}
\usepackage[algo2e,ruled,linesnumbered]{algorithm2e}

\usepackage{amsmath}
\usepackage{tikz}
\usetikzlibrary{positioning}
\usepackage{amsmath,amsfonts,amsthm}
\usepackage{thmtools}
\usepackage{xspace}
\usepackage{xcolor}
\usepackage{url}
\usepackage{caption}
\usepackage{subcaption}
\usepackage{tikz}
\usepackage{paralist}
\usepackage{libertine}
\usepackage{dsfont}
\usepackage[
pagebackref,
colorlinks=true,
urlcolor=blue,
linkcolor=blue,
citecolor=blue,
]{hyperref}
\usepackage[nameinlink]{cleveref}
\usepackage{enumitem}
\usepackage{todonotes}

\usepackage{color-edits}
\addauthor{pco}{purple}
\addauthor{jz}{olive}

\usepackage{thm-restate}

\usepackage[margin=1.05in]{geometry}
\usepackage{txfonts}
\usepackage[T1]{fontenc}

\theoremstyle{definition}
\theoremstyle{assumption}
 \newtheorem{definition}{Definition}[section]

\newtheorem{theorem}{Theorem}[section]
\newtheorem{lemma}[theorem]{Lemma}

\newtheorem{corollary}[theorem]{Corollary}

\newtheorem{assumption}{Assumption}

\newenvironment{proofof}[1]{\begin{proof}\textbf{of {#1}}}{\end{proof}}

\newcommand{\Acal}{\mathcal{A}}
\newcommand{\Bcal}{\mathcal{B}}

\newcommand{\Lcal}{\mathcal{L}}

\newcommand{\Ncal}{\mathcal{N}}

\newcommand{\Pcal}{\mathcal{P}}

\newcommand{\Tcal}{{\mathcal{T}}}

          \newcommand{\RR}{\mathbb{R}}

\def\dist{\textnormal{\texttt{dist}}}

\newcommand{\argmax}{\text{argmax}}

\newcommand{\eps}{\varepsilon}

\DeclareMathAlphabet\mathbfcal{OMS}{cmsy}{b}{n}

\newcommand{\R}{{\mathbb{R}}}

\newcommand{\conv}{\mathrm{conv}}

\DeclareMathOperator*{\argmin}{arg\,min}

\newcommand{\sign}{\mathrm{sign}}

\newcommand{\skipping}[1]{}

\newcommand{\offmax}{\operatorname{off-max}} 

\title{Efficient Opportunistic Approachability}

\author{Teodor Vanislavov Marinov \\
  \small \texttt{tvmarinov@google.com} \\
  \small Google Research
  \and
  Mehryar Mohri \\
  \small \texttt{mohri@google.com} \\
  \small Google Research
  \and
  Princewill Okoroafor \\
  \small \texttt{pco9@cornell.edu} \\
  \small Cornell University
  \and
  Jon Schneider \\
  \small \texttt{jschnei@google.com} \\
  \small Google Research
  \and
  Julian Zimmert \\
  \small \texttt{zimmert@google.com} \\
  \small Google Research
}
\date{}

\begin{document}
\pagenumbering{gobble}
\begin{titlepage}
\maketitle

\begin{abstract}
We study the problem of \emph{opportunistic approachability}: a generalization of Blackwell approachability where the learner would like to obtain stronger guarantees (i.e., approach a smaller set) when their adversary limits themselves to a subset of their possible action space. \cite{opportunistic} introduced this problem in 2014 and presented an algorithm that guarantees sublinear approachability rates for opportunistic approachability. However, this algorithm requires the ability to produce calibrated online predictions of the adversary's actions, a problem whose standard implementations require time exponential in the ambient dimension and result in approachability rates that scale as $T^{-O(1/d)}$.
In this paper, we present an efficient algorithm for opportunistic approachability that achieves a rate of $O(T^{-1/4})$ (and an inefficient one that achieves a rate of $O(T^{-1/3})$), bypassing the need for an online calibration subroutine. Moreover, in the case where the dimension of the adversary's action set is at most two, we show it is possible to obtain the optimal rate of $O(T^{-1/2})$. \end{abstract}
\end{titlepage}
\tableofcontents

\clearpage

\pagenumbering{arabic}
\section{Introduction}

Blackwell's approachability theorem has proven itself to be an incredibly valuable tool for tackling a wide range of different problems in machine learning, economics, and game theory involving repeated interactions with multi-objective constraints. At a high level, this theorem allows a learner repeatedly taking decisions in an adversarial environment to steer the average value of a vector-valued payoff function so that it ``approaches'' a certain set.

Slightly more formally, in the standard setting of Blackwell approachability, a learner and an adversary play a repeated game for $T$ rounds. Every round $t \in \{1, 2, \dots, T\}$, the learner selects an action $p_t$ from a convex action set $\Pcal$, and simultaneously the adversary selects an action $\ell_t$ from a convex action set $\Lcal$. Upon doing so, the learner receives a vector-valued payoff $u(p_t, \ell_t) \in \mathbb{R}^d$, for some bilinear function $u(\cdot, \cdot)$. The goal of the learner is to steer the average value $\frac{1}{T}\sum_{t}u(p_t, \ell_t)$ of this payoff so that it converges to a convex set $S$. If this is possible, the set is \emph{approachable}.

Echoing von Neumann's minimax theorem, Blackwell's approachability theorem states that a convex set $S$ is approachable if and only if for every adversary action $\ell \in \Lcal$, there exists a response action $p^{*}(\ell) \in \Pcal$ such that $u(p^{*}(\ell), \ell)$ lies in $S$. In particular, this theorem implies that for any response function $p^{*}: \Lcal \rightarrow \Pcal$, it is possible for the learner to guarantee that the time-average value of $u(p_t, \ell_t)$ converges to the convex hull $S = \mathrm{conv}\{u(p^{*}(\ell), \ell) \mid \ell \in \Lcal\}$. 

While Blackwell's theorem provides a sharp characterization of when a set can be approached in the worst-case sense, this guarantee can be overly restrictive in practice. In many situations, the adversary's behavior is not fully adversarial but instead follows some structured, stochastic, or cooperative pattern. This motivates the notion of \emph{opportunistic approachability} \citep{opportunistic}, which asks not whether the learner can enforce convergence to $S$ against all possible adversary strategies, but whether the learner can exploit opportunities to approach $S$ whenever the adversary's play makes this feasible. In particular, (one framing of) opportunistic approachability asks for a learning algorithm with the following ``data-dependent'' guarantee: if the adversary only plays actions in the convex set $Q \subseteq \Lcal$, then the learning algorithm should guarantee that the time-average value of $u(x_t, y_t)$ converges to the restricted set $S(Q) = \mathrm{conv}\{u(p^{*}(y), y) \mid y \in Q\}$. Importantly, the algorithm should accomplish this without any prior knowledge of the set $Q$.

In their paper originating this notion, \cite{opportunistic} provide an algorithm for opportunistic approachability based on the technique of best responding to calibrated predictions of the adversaries actions $y_t$. While this algorithm does satisfy the requirements of opportunistic approachability, the \emph{rates of convergence} it guarantees to the set $S(Q)$ are quite poor\footnote{\cite{opportunistic} do not directly study the convergence rates or computational complexity of their algorithm, and use the existence of online calibrated predictions as a black box. However, the standard algorithms for constructing $d$-dimensional online calibrated predictions produce predictions with average calibration error $T^{- O(1/d)}$ and are computationally inefficient, and the opportunistic approachability algorithms for \cite{opportunistic} inherit similar rates. We discuss this further in the related work.}. In particular, if $d_{\Lcal} = \dim(\Lcal)$ is the dimension of the adversary's action set, then the algorithm of \cite{opportunistic} only guarantees that the distance between $\frac{1}{T}\sum_{t}u(p_t, \ell_t)$ and $S(Q)$ is at most $T^{-O(1/d_{\Lcal})}$ -- in other words, it requires at least $T = \delta^{-\Omega(d_{\Lcal})}$ rounds to guarantee $\delta$-convergence to the set $S(Q)$. The algorithm is similarly computationally inefficient -- obtaining a convergence rate of $\delta$ requires maintaining a set of weights, one for each point in an $\delta$-discretization of the adversary's action space.  

\subsection{Our Results and Technical Overview}

In this paper, we present several new algorithms for opportunistic approachability with \emph{polynomial} rates of convergence -- i.e., that guarantee $\delta$-convergence to the set $S(Q)$ after at most $T = \mathrm{poly}(1/\delta)$ rounds. Many of these algorithms are also computationally efficient, running in $\mathrm{poly}(1/\delta)$ time per round. 

We distinguish between two variants of opportunistic approachability. In \emph{strict} opportunistic approachability (the variant defined above), we require approachability to the set $S(Q)$ only for convex sets $Q$ which contain \emph{every} adversary action $\ell_t$. In ($\eps$-)statistical opportunistic approachability, we require the stronger guarantee of approachability to the set $S(Q)$ for any convex set $Q$ which contains at least a $(1-\eps)$ fraction of all adversary actions $\ell_t$ (see Definition~\ref{def:finite-time-statistical} for a precise definition). 

We design algorithms for both settings of opportunistic approachability. Let $d_{\Pcal} = \dim(\Pcal)$ equal the dimension of the learner's action set. For strict opportunistic approachability, we first present an algorithm \textbf{(\Cref{thm:strict-eff})} achieving rate $O(d_{\Pcal}^{1/2}T^{-1/4})$ that runs in polynomial time per iteration, followed by an algorithm \textbf{(\Cref{thm:strict-ineff})} with improved rate $O(d_{\Pcal}^{1/3}T^{-1/3})$, at the cost of a worse computational complexity (possibly exponential in $d_{\Pcal}$). Similarly, for $\eps$-statistical opportunistic approachability, we first present a simple algorithm \textbf{(\Cref{thm:sto-eff})} achieving rate $O(d_{\Pcal}^{1/2}T^{-1/4} + \eps^{1/3} d_{\Pcal}^{1/3})$, followed by a slightly more complex algorithm \textbf{(\Cref{thm:sto-ineff})} achieving the improved rate $O(d_{\Pcal}^{1/3}T^{-1/3} + \eps^{1/3} d_{\Pcal}^{1/3})$.

In contrast to the strictly opportunistic setting (where the algorithm of \Cref{thm:strict-eff} both has polynomial rates and is computationally efficient), all algorithms for the statistical opportunistic approachability generally require time exponential in the payoff dimension $d$. However, under some assumptions it is possible to implement the $O(T^{-1/4})$-rate algorithm efficiently (for example, if the response function $p^{*}$ is Lipschitz) -- see Section~\ref{sec:easy-statistical} for additional details.

Our algorithms sidestep the need for a high-dimensional online calibration subroutine as used by \cite{opportunistic}. Instead, they employ a novel epoch-based framework that decomposes the learning task into two coupled online learners operating at different timescales. In the outer loop of this framework, the outer online learner outputs a \emph{direction vector} in payoff space (trying to continually learn a direction in which the current time-averaged payoff is maximally separated from a current ``target payoff'' guaranteed to be in the set $S(Q)$). In the inner loop of this framework, the inner online learner then produces actions to play that minimize the loss in the direction vector output by the online learner. 

Key to obtaining many of our results efficiently is the idea that we can relax the losses fed to both learners in this framework: in fact, all that is necessary is that we have 1. a \emph{consistent target function} that takes an interval of losses played by the adversary and outputs an element $U \in \mathbb{R}^d$ guaranteed to belong to the final target set, and 2. a \emph{valid loss rule} $F(p; \lambda) \in \mathbb{R}$ that upper bounds the largest possible payoff $\langle \lambda, u(p, \ell)\rangle$ in the direction $\lambda$ obtainable by playing $p$ against a worst-case loss $\ell$. For both settings of opportunistic approachability, simple choices for these two functions lead to $T^{-1/4}$ rates, whereas more sophisticated choices (combined with a more careful path-length based analysis of regret of the inner learner) lead to $T^{-1/3}$ rates. 

Finally, we consider the question of whether it is possible to achieve the optimal rate of $O(T^{-1/2})$, even in the strictly opportunistic setting. We provide some positive evidence that this might be the case by constructing efficient algorithms with this rate when either the dimension $d$ of the payoff is $1$ (\Cref{thm:one-dim}), or when the dimension $d_{\Lcal}$ of the adversary's action set is at most $2$ (\Cref{thm:two-dim}). Interestingly, the algorithm of \Cref{thm:two-dim} adopts a convex geometric approach from \cite{leme2018contextual} to prove that the myopic approach of assuming that the adversary will continue to play losses in the convex hull of the losses already played achieves this optimal rate (unfortunately, this approach breaks down for dimensions greater than $3$, and does not appear to produce any rates better than the existing $O(T^{-1/3})$ rate).

We conclude with two open questions:
\begin{enumerate}[leftmargin=*]
    \item Does there exist an efficient algorithm that achieves a $O(T^{-1/2})$ rate for strictly opportunistic approachability in high dimensions ($d > 3$)? Alternatively, can we prove a lower bound of $\Omega(T^{-1/3})$?
    \item  Is statistically opportunistic approachability fundamentally harder than strictly opportunistic approachability? Does there exist a computationally efficient algorithm for statistically opportunistic approachability that achieves polynomial rates of convergence?\end{enumerate}

\subsection{Related Work}

\paragraph{High-Dimensional Calibration} The original algorithms of \cite{opportunistic} for opportunistic approachability require a subroutine for calibrated predictions of the adversary's actions. Such a subroutine was first shown to exist by \cite{foster1998asymptotic}, who used this to construct the first algorithms for minimizing swap regret (and hence converging to correlated equilibria) in games. This algorithm suffers from an exponential dependence on the dimension $d$ of the adversary's action set (i.e., producing $\delta$-calibrated predictions required at least $T = \delta^{-\Omega(d)}$ rounds). For a while, it was an open question whether this exponential dependence was necessary \citep{abernethy2011does}; this was largely resolved by recent work showing that although it is possible to construct algorithms producing $\delta$-calibrated predictions in $T = d^{\Omega(1/\delta)}$ rounds, there is no such algorithm with $T$ polynomial in both $d$ and $1/\delta$ \citep{peng2025high, fishelson2025high}.

\paragraph{Other Algorithms for Blackwell Approachability} We briefly survey here other known algorithms for solving the Blackwell approachability problem and how they relate to the two-stage framework we introduce in this paper.

Perhaps the most common technique for constructing an online learning algorithm for the approachability problem is the reduction to online linear optimization popularized by \cite{AbernethyBartlettHazan2011}, also implicit in the original construction of \cite{blackwell} and several follow-up works \citep{mannor2013approachability, Kwon2016, DannMansourMohriShneiderSivan2023}. This reduction relies on a sub-algorithm that produces a low-regret sequence $H_{t}$ of half-spaces containing the target set $S$, and then playing actions $p_t$ that guarantee that each individual payoff $u(p_t, \ell)$ lies in $H_{t}$ (regardless of the choice of adversary action $\ell$). In some ways this is analogous to the outer loop of our framework (which also chooses directions in utility space). One major difference, however, is that the outer loop of our procedure produces \emph{directions} in utility space, whereas this sub-algorithm produces \emph{half-spaces} (directions with achievable target values); this distinction is important, as we cannot tell whether a given half-space is valid until we have identified the final target set $S(Q)$ (the inner learner in our framework can perhaps be thought of as trying to learn the optimal half-space in a given direction).

Another method for solving approachability instances is provided by \cite{bernstein15a} in their study of \emph{response-based approachability} (i.e., approachability instances defined in terms of a response function $p^{*}: \Lcal\rightarrow\Pcal$). The algorithm they propose avoids the need to do online linear optimization over the set of half-spaces containing the target set (a task that can be computationally challenging \citep{bernstein15a, daskalakis2025efficient}). Similar to the outer loop of our framework, their algorithm identifies an approachability direction by looking at the displacement between the current time-average payoff and some target payoff guaranteed to be in $S$. Like the \cite{AbernethyBartlettHazan2011} reduction, this does not immediately generalize to opportunistic approachability due to the difficulty of guaranteeing a specific action $p$ is optimal for a given (scalar) payoff function without knowledge of the adversary's set of actions $Q$.

\paragraph{Applications of (Opportunistic) Approachability}

As mentioned, Blackwell approachability is an incredibly powerful tool for the design of online algorithms. Among many other applications, in recent years this has allowed researchers to design learning algorithms with strong strategic robustness properties in general games \citep{arunachaleswaran2025swap}, online prediction algorithms with strong downstream guarantees \citep{okoroafor2024faster, okoroafor2025near}, fair online learning algorithms \citep{chzhen2021unified}, and reinforcement learning algorithms satisfying multi-objective constraints \citep{MiryoosefiBrantleyDaumeDudikSchapire2019}.

Any application of Blackwell approachability can in theory benefit from the additional guarantees provided by opportunistic approachability. We highlight two specific applications here for which these additional guarantees are particularly meaningful.

\begin{itemize}[leftmargin=*]
    \item \textbf{Approximately approaching unapproachable sets}: It is not uncommon to have an online learning problem with a set of desired constraints that are not possible to guarantee simultaneously (i.e., a target set that is not in general approachable). This is the case for example in fair online learning: \cite{chzhen2021unified} present three fairness desiderata that they prove are not simultaneously attainable. \cite{MannorPerchetStoltz2014Bis} provide examples of other approachability instances where this is the case.

    In such cases, one may wish to approach the target set as closely as possible. By setting the response function $p^{*}(\ell)$ to be the response to $\ell$ that minimizes the distance to the target set, opportunistic approachability guarantees that one does approach the target set if the adversary limits themselves to a ``feasible'' convex subset of strategies $Q$.
    \item \textbf{Stronger profile swap regret guarantees}: \cite{arunachaleswaran2025swap} recently introduced the notion of \emph{profile swap regret} in general linear games, an efficiently minimizable regret notion with the property that any no-profile-swap-regret player is ``non-manipulable'' (no opponent can extract more per-round utility from them than they could obtain against a rational opponent in a single-shot game). Interestingly, minimizing profile swap regret corresponds exactly to the response-based approachability problem where $u(p, \ell) = p \otimes \ell$ and $p^{*}(\ell)$ is the best response to the action $\ell$. Running an opportunistic approachability algorithm therefore provides stronger non-manipulability guarantees in settings where the opponent's true action set is unknown.
\end{itemize} \section{Model and Preliminaries}
\subsection{Notation}

We write $[T]$ to denote the set $\{1, 2, \dots, T\}$, $\| \cdot \|$ to denote the $\ell_{2}$ norm, and $B_{d}(r) \subset \mathbb{R}^d$ to denote the $d$-dimensional $\ell_2$ ball of radius $r$.  The support function for a set $S$ at $\lambda \in \R^d$ is denoted by $\sigma_S(\lambda) = \sup_{s \in S} \langle s, \lambda \rangle$. We note that the distance of a point $x \in \R^d$ to $S$ can be written as $\texttt{dist}(x, S) = \sup_{\lambda \in B_d(1)} \langle \lambda, x \rangle - \sigma_S(\lambda)$.

\subsection{Blackwell Approachability and Opportunistic Approachability}

Blackwell's approachability theorem \citep{blackwell} extends repeated two-player zero-sum games to settings involving vector-valued payoffs. In this model, at each timestep $t=1, 2, \dots, T$ of the game, the learner selects an action $p_t \in \Pcal \subset \mathbb{R}^{d_{\Pcal}}$, and the adversary selects an action $\ell_t \in \Lcal \subset \mathbb{R}^{d_{\Lcal}}$. The learner then receives a vector-valued payoff $u(p_t, \ell_t)$, for some bi-affine payoff function $u: \Pcal \times \Lcal \to \RR^d$. We will assume the action sets $\Pcal$ and $\Lcal$ are compact and convex subsets, and that $u$ is bounded so that $\|u(p,\ell)\|\leq 1$ for any $p \in \Pcal$ and $\ell \in \Lcal$. We will further assume that the action set $\Pcal$ of the learner is in isotropic position (and in particular assume that $B_{d_{\Pcal}}(1) \subseteq \Pcal \subseteq B_{d_{\Pcal}}(d_{\Pcal})$, where $B_{d}(r)$ denotes the $d$-dimensional ball of radius $r$ centered at the origin).

\sloppy The learner's goal is to select actions $\{p_t\}_{t=1}^T$ (adaptively, choosing $p_t$ based on the current history $(p_1, \ell_1, \dots, p_{t-1}, \ell_{t-1})$) such that their average payoff vector converges to a given closed convex target set $S \subseteq \RR^d$. This convergence must be guaranteed regardless of the adversary's choices $\{\ell_t\}_{t=1}^T$. For a given closed convex set $S$, the learner aims to ensure that:
\begin{equation} \label{eq:approach}
  \dist \left(\frac{1}{T} \sum_{t=1}^T u(p_t, \ell_t), S \right)  \rightarrow 0 \quad \text{as} \quad T \rightarrow \infty
\end{equation}

A closed convex set $S$ is \emph{approachable} if the learner can guarantee this convergence property against any strategy of the adversary. Blackwell's theorem provides a fundamental characterization of approachable sets: a closed convex set $S \subseteq \RR^d$ is approachable if and only if there exists a \emph{response function} $p^* : \Lcal \rightarrow \Pcal$ such which satisfies $u(p^*(\ell), \ell) \in S$. 

\sloppy This characterization motivates the framework of \emph{response-based approachability} \citep{bernstein15a}, where we define an approachability instance via a response function $p^*: \Lcal \rightarrow \Pcal$. Given such a function, the goal of the player is to approach the convex hull of the set of responses to each $\ell \in \Lcal$, that is to approach the set
\begin{align*}
    S (\Lcal) = \mathrm{conv}\{u(p^*(\ell), \ell) \mid \ell \in \Lcal\}
\end{align*}

\emph{Opportunistic approachability} \citep{opportunistic} is a refinement of response-based approachability that also begins with a desired response function $p^* : \Lcal \rightarrow \Pcal$ for the learner, but where the learner's target set is contingent on the adversary's actual behavior over time. For any closed, convex set $Q$, define

$$ S (Q) = \mathrm{conv}\{u(p^*(\ell), \ell) \mid \ell \in Q\} $$
The goal of the learner in opportunistic approachability is to guarantee that, if $Q$ is a convex set containing (asymptotically) all points of the adversary's action sequence, then the average payoff vector $\frac{1}{T} \sum_{t=1}^T u(p_t, \ell_t)$ converges to the set $S(Q)$. That is, the learner is ``opportunistic'' – they steer their average payoff to the smallest achievable set $S (Q)$ that is consistent with the adversary's actual, possibly restricted, play. Note that $S (Q) \subseteq S (\Lcal)$, so the target is indeed potentially smaller (better) when the adversary restricts their play ($Q \subseteq \Lcal$). Formally, \cite{opportunistic} define the objective of opportunistic approachability as follows.

\begin{definition}[Statistically/Strictly $Q$-Restricted Adversary \citep{opportunistic}]\label{def:limit-statistical}
Let $Q$ be a closed convex subset of $\Lcal$.
We say that the sequence of pure actions $\{\ell_t\}_{t=1}^\infty$ played by an adversary is \textit{statistically} $Q$-restricted if the average distance of the actions to $Q$ converges to zero:
\[
\lim_{T \to \infty} \frac{1}{T} \sum_{t=1}^T \dist(\ell_t, Q) = 0.
\]
We say the adversary's play is \textit{strictly} $Q$-restricted if $\ell_t \in Q$ for all $t$.
\end{definition}

\begin{definition}[Statistically/Strictly Opportunistic Learner Strategy \citep{opportunistic}]
Let $p^*: \Lcal \to \Pcal$ be a desired response function.
We say that a learner's strategy is \textit{statistically opportunistic} with respect to $p^*$ if, for any $Q \subseteq \Lcal$ and any sequence of pure actions $\{\ell_t\}_{t=1}^\infty$ played by a statistically $Q$-restricted adversary, the learner's strategy guarantees that the average payoff converges to the set $S (Q) = \mathrm{conv}\{u(p^*(\ell), \ell) \mid \ell \in Q\}$. That is, 
\[
\lim_{T \to \infty} \dist \left(\frac{1}{T} \sum_{t=1}^T u (p_t, \ell_t), S (Q) \right) = 0,
\]
Similarly, a learner's strategy is \textit{strictly opportunistic} with respect to $p^*$ if this guarantee holds against any strictly $Q$-restricted adversary.

\end{definition}

\subsection{Finite Time Opportunistic Approachability}
\cite{opportunistic} define their notion of statistically $Q$-restricted opportunistic approachability only in the limit as $T\rightarrow\infty$. 
In this paper, we are interested in finite-time guarantees, and therefore need to modify these definitions for finite time horizons. It is straightforward to adapt the definition of a strictly opportunistic learner strategy to the finite time setting by simply requiring every adversary action $\ell_t$ to lie in the restriction set $Q$. Formally:

\begin{definition}[Finite-time Strictly Opportunistic Learner]\label{def:finite-time-strict}
We say an algorithm is strictly opportunistic at a rate of $\gamma(T)$ if, for any sequence of adversary actions $\{\ell_t\}_{t=1}^T$, its average payoff vector $\bar u = \frac{1}{T} u(p_t, \ell_t)$ satisfies
\[
\dist \left(\bar u_T, S(Q) \right) \leq \gamma(T)
\]
for $Q = \conv(\{\ell_{t}\}_{t\in [T]})$.
\end{definition}

On the other hand, there is some subtlety in how exactly to adapt the limit-based definition of a statistically $Q$-restricted adversary to finite time. For example, the following lemma shows that a straightforward translation of Definition~\ref{def:limit-statistical} to finite time is not possible\footnote{Intuitively, the issue is that the response function $p^{*}$ can be discontinuous, causing two $\eps$-restricted sets to have very different target sets.}.

\begin{lemma}
There exists an approachability problem instance such that for any learner, any $T>0$, and any $\eps >0$,  there exists an oblivious adversary $\{\ell_t\}_{t=1}^{T}$ and convex set $Q$ such that
$
    \frac{1}{T}\sum_{t=1}^T\dist(\ell_t, Q)\leq \eps,
$
but $\dist(\bar{u}_T, S(Q))\ge \frac{1}{2}$.
\end{lemma}
\begin{proof}
Define $\Pcal = \Lcal = [-2, 2]$, the utility to be $u(p,\ell) = p\ell \in \mathbb{R}$, and define the response function via 

$$p^{*}(\ell) = \begin{cases}-1 & \mbox{if } \ell \in [-2, 1)\\1 & \mbox{if } \ell \in [1, 2]\end{cases}$$
\noindent
The adversary will pick the loss sequence $\ell_t = 1$ for all $t \in [T]$.

Now consider the (singleton) sets $Q_-=\{1-\eps\}$ and $Q_{+}=\{1+\eps\}$. Note that both $Q{-}$ and $Q_{+}$ satisfy $\frac{1}{T}\sum_{t=1}^T\dist(\ell_t, Q)\leq \eps$. However, note that $S(Q_{-}) = \{-1\}$ and $S(Q_+) = \{1\}$, and so at least one of $\dist(\bar{u}_T, S(Q_{-}))$ and $\dist(\bar{u}_T, S(Q_{+}))$ must be larger than $1/2$. 
\end{proof}

To work around this we consider a stronger, but still natural, notion of $Q$-restrictedness, which will guarantee the existence of a statistically opportunistic learner strategy for sufficiently small $\eps$. We do this by considering the closed convex subsets of $\Lcal$ which contain all but an $\eps$-fraction of $\{\ell_t\}_{t \in [T]}$, that is $\frac{1}{T} \sum_{t=1}^T \mathbb{I}[\ell_t \notin Q]  \leq \eps$. Helly's theorem \citep{danzer1963helly} then implies that for sufficiently small $\eps = O(1/d)$, the intersection of all such subsets will be non-empty (\Cref{lem:max-points-outside-Q}). Let this intersection be $Q_{\mathrm{int}}^\eps$. If $Q_{\mathrm{int}}^\eps$ were known to the learner, then  employing standard Blackwell approachability algorithms to approach the set $S(Q_{\mathrm{int}}^\eps)$ (ignoring rounds where $u(p^{*}(\ell_t), \ell_t) \not\in S(Q_{\mathrm{int}}^\eps)$) guarantees approachability to $S(Q_{\mathrm{int}}^\eps)$ with a rate of $O(1/\sqrt{T})$. Next, we formally define this intersection and the target set for the approachability problem.

\begin{definition}[Adversary Q-sets]\label{def:qsets}
Given a sequence of adversary losses $\{\ell_t\}_{t\in \Tcal}$ and a parameter $\eps \ge 0$, let $\Tcal'$ be any subset of $\Tcal$ such that its size $|\Tcal|$ is at least $(1-\eps)|\Tcal|$. Let
\begin{align}
    \mathbf{Q}^\eps (\{\ell_t\}_{t \in \Tcal}) = \left\{ \conv(\{\ell_t\}_{t \in \Tcal'}) \mid \Tcal' \subseteq \Tcal \land |\Tcal'| \ge (1-\eps)|\Tcal| \right\}
\end{align}
\end{definition}

\begin{definition}[Intersection of Adversary Q-sets]
\label{def:intersection_core}
Given a sequence of adversary actions $\{\ell_t\}_{t\in\Tcal}$ and a parameter $\eps \ge 0$:
The intersection set, $Q_{\mathrm{int}}^\eps(\{\ell_t\}_{t\in\Tcal})$, is the intersection of all possible adversary Q-sets:
\begin{align*}
    Q_{\mathrm{int}}^\eps(\{\ell_t\}_{t\in\Tcal}) = \bigcap_{Q \in \mathbf{Q}^\eps(\{\ell_t\}_{t\in\Tcal})} Q
\end{align*}
\end{definition}

\begin{definition}[Opportunistic Target Set]
\label{def:opportunistic_target}
Given the intersection set $Q_{\mathrm{int}}^\eps(\{\ell_t\}_{t\in\Tcal})$ from Definition~\ref{def:intersection_core}, the opportunistic target set is the set of ideal payoffs corresponding to the actions within the intersection set:
$$S^\eps_\mathrm{int} (\{\ell_t\}_{t\in\Tcal}) = S\left(Q_{\mathrm{int}}^\eps (\{\ell_t\}_{t\in\Tcal})\right) = \mathrm{conv}\{u(p^*(\ell), \ell) \mid \ell \in Q_{\mathrm{int}}^\eps (\{\ell_t\}_{t\in\Tcal}) \}.$$
\end{definition}
When we take the Opportunistic Target Set over the full sequence $(\{\ell_t\}_{t=1}^T)$, the explicit dependence on this will be omitted for brevity (e.g., $S^\eps_\mathrm{int}$).

\sloppy Note that aiming for $S(Q_{\mathrm{int}}^\eps)$ is a stronger objective than what is strictly required for an opportunistic strategy. A sufficient condition would be to guarantee that the average payoff lies in the intersection of all target sets, $\bigcap_{Q \in \mathbf{Q}^\eps} S(Q)$. This is the approach taken in \cite{bernstein15a}. However, since $Q_{\mathrm{int}}^\eps \subseteq Q$ for all $Q \in \mathbf{Q}^\eps$, it follows that $S(Q_{\mathrm{int}}^\eps) \subseteq S(Q)$. Therefore, $S(Q_{\mathrm{int}}^\eps) \subseteq \bigcap_{Q \in \mathbf{Q}^\eps} S(Q)$. By successfully approaching the smaller set $S(Q_{\mathrm{int}}^\eps)$, a learner automatically guarantees closeness to every set $S(Q)$ for $Q \in \mathbf{Q}^\eps$. 

\begin{definition}[Finite-time Statistically Opportunistic Learner]\label{def:finite-time-statistical}
Given a parameter $\eps \ge 0$, we say an algorithm is $\eps$-statistically opportunistic at a rate of $\gamma(T)$ and error rate $\alpha(\eps)$ if, for any sequence of adversary actions $\{\ell_t\}_{t=1}^T$, its average payoff vector $\bar u_{T} = \frac{1}{T} \sum_{t} u(p_t, \ell_t)$ satisfies
\[
\dist \left(\bar u_T, S^\eps_\mathrm{int} \right) \leq \gamma(T)+\alpha(\eps).
\]
\end{definition}
 \subsection{Online Optimization Oracles}
Our algorithm will rely on  online optimization learners as subroutines.

\begin{definition}[Online Optimization (OO) Oracle]
An OO oracle is an algorithm that operates over a convex action set $\mathcal{A}$. It interacts with an environment over $T$ rounds as follows: in each round $t=1, \dots, T$, the oracle first chooses an action $x_t \in \mathcal{A}$, and then receives a loss function $f_t: \mathcal{A} \to \mathbb{R}$. The oracle provides guarantees on the regret
$$
\mathrm{Reg}_T = \sum_{t=1}^T f_t(x_t) - \min_{x \in \mathcal{A}} \sum_{t=1}^T f_t(x)
$$
Typically, this guarantee is on the order of $\mathrm{Reg}_T = O(\sqrt{T})$.
If the loss functions $f_t$ are convex, we call this an Online Convex Optimization (OCO) oracle.
\end{definition}

One important example of an OO Oracle is that of \emph{Online Gradient Descent}.
\begin{lemma}[OCO Guarantee Theorem 3.1 of \cite{hazan2023}]
\label{lem:oco_guarantee}
Online Gradient Descent implements an online convex optimization oracle with a regret bound that scales as $O \left(GD\sqrt{T} \right)$, where $G$ is a uniform bound on the norm of the gradients of the loss functions (i.e., $\|\nabla f_t(x)\| \le G$) and $D$ is the diameter of the convex action set $\mathcal{A}$.
\end{lemma}
\iffalse
\begin{lemma}[OCO guarantee Theorem 3.1 of \cite{bubeck2011introduction}]
\label{lem:oco cont exp}
Continuous Exponential Weights (CEW) implements an online convex optimization oracle for loss functions bounded in $[-1,1]$ with a regret bound that scales as
$O\left(\sqrt{d_{\Acal}T\log(T)}\right)$ where $d_{\Acal}$ is the dimension of the action set $\mathcal{A}$.
\end{lemma}

\fi \section{An Algorithmic Framework for Opportunistic Approachability}
\label{sec:framework}
The algorithms in this paper share a common structure. This section introduces a general, epoch-based framework that serves as a blueprint for these algorithms. The framework is epoch-based and operates over $L$ epochs, each lasting $T_e = T/L$ timesteps. It involves two collaborating online learners: an outer learner and an inner  learner. The outer learner operates on the slow timescale of epochs (i.e each timestep corresponds to a single epoch). At the start of each epoch $e$, it chooses a direction vector, $\lambda_e \in \mathcal{B}_d(1)$. This vector defines the objective for the inner learner. A new inner learner is instantiated at the start of each epoch with a fixed direction $\lambda_e$ from the outer learner and its goal is to choose actions $p_t$ to minimize a sequence of loss functions derived from $\lambda_e$.
At the end of an epoch, the combination of the adversary and learner's choices is used to generate a loss for the outer learner, which in turn updates its strategy to select a new direction $\lambda_{e+1}$. To make this framework adaptable, we introduce two abstract components: a \textit{Loss Rule} for the inner learner and a \textit{Target Function} for the outer learner.

A \textit{Loss Rule} $F$ is a mapping that takes a history of adversary actions $(\ell_{t'})_{t'\leq t}$ and a direction vector $\lambda$ and outputs a loss function $f_t: \Pcal \to \mathbb{R}$. For the framework to be valid, we require any such rule to produce losses that are upper bounds on the true linear loss.

\begin{assumption}[Valid Loss Rule]
\label{ass:loss_upper_bound}
A loss rule $F$ is valid if for any history and any $\lambda \in \mathcal{B}_d(1)$, the resulting loss function $f_t = F((\ell_{t'})_{t'\leq t}; \lambda)$ satisfies:
$
\forall p \in \Pcal, \quad f_t(p) \ge \langle \lambda, u(p, \ell_t) \rangle
$.
\end{assumption}
The simplest example of a valid loss rule is $F((\ell_{t'})_{t'\leq t}; \lambda) = \langle \lambda, u(\cdot, \ell_t) \rangle$, which satisfies the assumption with equality.

Similarly, we abstract the process of generating a target for the outer learner. A \textit{Target Function} $U$ is a mapping that takes a history of adversary actions observed during an epoch, $\mathcal{L}_e = (\ell_t)_{t \in \mathcal{T}_e}$, the direction $\lambda$, and outputs a target payoff vector $U(\mathcal{L}_e;\lambda_e)=u^*_e \in \mathbb{R}^d$.
For this target to be meaningful, it must be \textit{consistent} with the overall goal of the algorithm, which is to approach some true opportunistic target set $S_\mathrm{int}^\eps$.

\begin{definition}[Consistent Target Function]
\label{def:consistent_target}
A target function $U$ is $\eps$-\textbf{consistent} if, for any epoch $e$, its output $U(\mathcal{L}_e;\lambda_e)$ is guaranteed to be an element of the final target set, i.e., $U(\mathcal{L}_e;\lambda_e) \in S_\mathrm{int}^\eps(\{\ell_t\}_{t=1}^T)$.
\end{definition}
The inner learner needs to be able to minimize its losses relative to the target and the direction vector provided by the outer learner.

\iffalse
\begin{definition}[Inner Error Term]
\label{def:error_term}
The error, $\mathrm{err}(F, U)$, of a given Loss Rule $F$ and Target Function $U$ is the tightest possible upper bound on the inner loop's objective gap, holding for any sequence of adversary actions:
\[
\frac{1}{T_e} \min_{p \in \Pcal} \sum_{t \in \mathcal{T}_e} f_t(p) - \langle \lambda_e, u^*_e \rangle \le \mathrm{err}(F, U)
\]
where $f_t = F(\dots; \lambda_e)$ and $u^*_e = U(\mathcal{L}_e)$.
\end{definition}

This error term captures how well the best fixed action $p^*$ \jzcomment{$p^*$ is the response function, I want to avoid overloading it to denote a specific action as well.} could have performed on the surrogate losses $f_t$ relative to the epoch's target $u^*_e$. A smaller error term indicates that the chosen $F$ and $U$ are well-aligned.
\fi
\paragraph{Remark.} A computationally simple choice in the strict case ($\eps=0$) for the target function is $U(\mathcal{L}_e) = u(p^*(\bar{\ell}_e), \bar{\ell}_e)$, where $\bar{\ell}_e$ is the average action in the epoch. While this is consistent, this may not minimize the error term. The theoretically optimal choice for a given $\lambda_e$ given the largest set $Q_e$ such that $Q_e\subset Q_{\text{int}}^\eps$ would be $u^*_e = \arg\max_{u \in S(Q_e)} \langle \lambda_e, u \rangle$, however, it may be difficult to compute this efficiently.

\paragraph{Overview of \Cref{alg:general}}
The algorithm is an epoch-based algorithm that involves an outer OCO (Online Convex Optimization) learner, an inner OO (Online Optimization) learner, a loss rule $F$ and a target function $U$. There are $L$ epochs, each lasting $T_e=T/L$. The outer OCO runs over $L$ timesteps, one for each epoch. A new inner OO is initialized for each epoch. The outcome of the interaction in the inner OO is then used to construct the loss vector for the outer OCO.

Specifically, the outer oracle operates in the dual space of payoffs, $\mathcal{B}_d(1)$, and selects a direction vector $\lambda_e$ for each epoch. This direction is passed to the inner oracle, which plays the game against the adversary in the learner's action space $\Pcal$. The inner oracle's objective for the epoch is simply to choose actions $p_t$ that minimize the loss sequence $f_t$ generated by the loss rule, which are all upper bounds on the scalar loss projected along the given direction, $\langle \lambda_e, u(p_t, \ell_t) \rangle$.

At the end of an epoch the outer oracle's loss vector is the aggregated average difference between the realized utility and this adaptive target, $g_e = \frac{1}{T_e}\sum_{t \in \mathcal{T}_e} (u(p_t, \ell_t) - u^*_e)$. This loss informs the outer oracle's choice of a new direction for the next epoch, repeating the cycle.

\begin{algorithm}[H]
\caption{Epoch-based Opportunistic Learner}
\label{alg:general}
\KwIn{OO oracles $\texttt{OO}_{\text{Inner}}$, $\texttt{OCO}_{\text{Outer}}$, target function $U$, loss rule $F$, number of epochs $L$}
\BlankLine
Initialize epoch length $T_e = T/L$\;

Query $\texttt{OCO}_{\text{Outer}}$ on $\mathcal{B}_d(1)$ for initial direction $\lambda_1$\;

\For{epoch $e \gets 1$ \KwTo $L$}{
    Initialize fresh $\texttt{OO}_{\text{Inner}}$ on action space $\mathcal{P}$\;
    \For{$t \gets (e-1)T_e+1$ \KwTo $eT_e$}{
        \tcp{Inner loop plays against the adversary}
        Play $p_t$ from $\texttt{OO}_{\text{Inner}}$ and observe $\ell_t$\;
        
        Receive $f_t = F((\ell_{t'})_{t'\leq t};\lambda_e)$\;
        
        Update $\texttt{OO}_{\text{Inner}}$ with loss $f_t$\;
    }
    \tcp{Outer loop updates the direction vector}
    
    Compute epoch target $u^*_e = U((\ell_t)_{t\in\mathcal{T}_e};\lambda_e)$\;
    
    Define outer loss $g_e = \tfrac{1}{T_e}\sum_{t\in\mathcal{T}_e} \big(u(p_t,\ell_t)-u^*_e\big)$\;
    
    Update $\texttt{OCO}_{\text{Outer}}$ with linear loss $\langle \cdot, g_e \rangle$ to get $\lambda_{e+1}$\;
}
\end{algorithm}

\begin{theorem}
\label{thm: general}
Given a loss rule $F$ satisfying \Cref{ass:loss_upper_bound} and an $\eps$-consistent target function $U$, \Cref{alg:general} satisfies 
\begin{align*}
    \dist(\bar u_T, S_{\text{int}}^\eps) \leq  \frac{\mathrm{Reg}_{T_e}^{\text{in}}}{T_e}+\frac{\mathrm{Reg}_{L}^{\text{out}}}{L} + \rm{err}\,,
\end{align*}
where $\rm{err}=\max_{e\in[L]}\left(\frac{1}{T_e} \min_{p \in \Pcal} \sum_{t \in \mathcal{T}_e} f_t(p) - \langle \lambda_e, u^*_e \rangle\right)$.
\end{theorem}
\begin{proofof}[\Cref{thm: general}]:
let $\lambda^* \in \mathcal{B}_d(1)$ be the direction vector defining the distance, such that $\dist(\bar{u}_T, S_T) = \langle \lambda^*, \bar{u}_T \rangle - \sigma_{S_T}(\lambda^*)$, where $\sigma_S(\lambda)$ denotes the support function of $S$ at $\lambda$. The proof bounds the total deviation, $T \cdot \dist(\bar{u}_T, S_T)$, by chaining the regret guarantees of the two oracles.

From the consistency of the target function (Definition~\ref{def:consistent_target}) we have $u_e^* \in S_T$, which implies $\langle \lambda^*, u_e^* \rangle \le \sigma_{S_T}(\lambda^*)$. This allows us to bound the total deviation as follows:
\begin{align*}
T \cdot \dist(\bar{u}_T, S_T) &= \langle \lambda^*, T\bar{u}_T \rangle - T \sigma_{S_T}(\lambda^*) \\
&\le \sum_{e=1}^L \sum_{t \in \mathcal{T}_e} \langle \lambda^*, u(p_t, \ell_t) - u_e^* \rangle && \text{(Since $u_e^* \in S_T$)}\\
&\le \sum_{e=1}^L \sum_{t \in \mathcal{T}_e} \langle \lambda_e, u(p_t, \ell_t) - u_e^* \rangle + T_e\text{Reg}_L^{\text{out}} && \text{(Outer Oracle Regret)}\\
&\le \sum_{e=1}^L \sum_{t \in \mathcal{T}_e} f_t(p_t) - \langle \lambda_e, u_e^* \rangle + T_e\text{Reg}_L^{\text{out}} && \text{(\Cref{ass:loss_upper_bound})}\\
&\le \sum_{e=1}^L\Big( \min_{p^*\in\mathcal{P}} \sum_{t\in\mathcal{T}_e}(f_t(p^*)-\langle \lambda_e,  u_e^* \rangle)+\text{Reg}_{T_e}^{\text{in}}\Big) + T_e\text{Reg}_L^{\text{out}} && \text{(Inner Oracle Regret)}\\
&\le \rm{err}\cdot T +L\text{Reg}_{T_e}^{\text{in}} + T_e\text{Reg}_L^{\text{out}} && \text{}
\end{align*}
Dividing both sides by $T$ completes the proof.
\end{proofof}
Using Online Gradient Descent as the outer learner yields the following corollary.
\begin{corollary}\label[corollary]{cor:general}
    In the same setup as \Cref{thm: general}, if the outer OCO is an instance of Online Gradient Descent with optimal tuning, the guarantee is
    
\begin{align*}
    \dist(\bar u_T, S_{T}) \leq  \frac{1}{T_e}\mathrm{Reg}_{T_e}^{\text{in}}+O(L^{-\frac{1}{2}}) + \rm{err}
\end{align*}
\end{corollary}

\subsection{Strictly Opportunistic Approachability}
As a warm-up, we present a simple and efficient instantiation of \Cref{alg:general} that results in a strictly opportunistic learning algorithm. Let $f_t = F(\dots; \lambda_e)$ and $u^*_e = U(\mathcal{L}_e)$ be defined by
\begin{align}\label{eq: strict T^3/4}
f_t(p) = \langle \lambda, u(p,\ell_t)\rangle,\qquad u^*_e = u(p^*(\bar\ell_e), \bar\ell_e)\,.
\end{align}
The choice of $f_t$ is the typical loss used by an OCO oracle in the classical reduction of Blackwell approachability to regret minimization while the choice of $u^*$ is just the best response to the average play of the adversary in an epoch. \Cref{alg:general} instantiated in this way is computationally efficient as long as projecting onto the learner's set, $\Pcal$, is efficient and querying the response function, $p^*$, is efficient.

\begin{theorem}\label{thm:strict-eff}[Efficient Strictly Opportunistic Learner]
 Instantiating \Cref{alg:general} with the loss rule and target function defined in \Cref{eq: strict T^3/4} and OGD for both OO oracles yields a strictly opportunistic learning algorithm with rate
 $\gamma(T)=\tilde{O}(d_{\Pcal}^{\frac 1 2}T^{-\frac{1}{4}})$.
\end{theorem}
The proof is deferred to the \Cref{app: thm proofs 3/4}. It shows that the $\rm{err}$ term is non-positive in the strict case, applies the $\sqrt{T_e}$ OGD bound to the inner regret and tunes the parameters.

Our rates have a dependency on the dimension of the players action set, which is absent in regular approachability. The dimension dependency can be improved at the cost of reduced computational efficiency (e.g. using continuous exponential weights \citep{bubeck2011introduction} as the inner OO), but some dependency is unavoidable for the opportunistic approachability problem, as the following lower bound shows.
\begin{lemma}\label[lemma]{lem: lower d bound}
For any learner, there exists an approachability problem for action set $\Pcal,\Lcal\subset\mathbb{R}^d$, $u(\Pcal,\Lcal)\subset [-1,1]$ and a sequence of losses $(\ell_t)_{t=1,\dots d}$ such that $\dist(\bar u_{d}, S_{\rm{int}}^0(\{\ell_t\}_{t\in[d]}))=\Omega(1)$.
\end{lemma}
\begin{proofof}[\Cref{lem: lower d bound}]
    We let the action set $\Pcal, \Lcal$ be the $\ell_\infty$ and $\ell_1$ balls in $d$-dimensions respectively.
    The payoff function is the inner product of players and adversaries action. The response function is $p^*(\ell)_i=\sign(\ell_i)$.
    Let the adversary pick the sequence, where they randomly select $\pm\mathbf{e}_t$ for $t=1,\dots,d$.
    The expected average payoff vector for any learner is $0$, but $S^0_{\text{int}}=\{1\}$.
    
\end{proofof}
Any rate that only depends of the maximum norm of payoffs would violate this lower bound.

\subsection{Statistically Opportunistic Approachability}\label{sec:easy-statistical}
We now present an instantiation of \Cref{alg:general} that yields a statistically opportunistic learner.
The target function of \Cref{eq: strict T^3/4} is not always consistent in this case, since $\bar\ell_e$ is not guaranteed to be in $Q_{\text{int}}^\eps$.
We solve this problem by finding a ``close'' point to $\bar\ell_e$ that is guaranteed to be in $Q^\eps_{\text{int}}$, where closeness is defined in the following way using the total variation of a discrete distribution $\rm{TV}(\alpha,\rm{unif}):=\sum_{t\in\mathcal{T}_e} |\alpha_t-\frac{1}{T_e}|$
\begin{align*}
    \alpha^*_e = \argmin_{\substack{\alpha\in \Delta(\mathcal{T}_e)\\ \sum_{t\in\mathcal{T}_e}\alpha_t\ell_t \in Q^{\eps}_e}} \rm{TV}(\alpha,\rm{unif})\,,\qquad \bar\ell_e^* = \sum_{t\in\mathcal{T}_e}\alpha_t\ell_t\,.
\end{align*}
This is well defined with bounded divergence.
\begin{lemma}\label[lemma]{lem:bound_sto_mean}
For $T_e > d_{\Pcal} \eps$, $\alpha^*_e$ exists and satisfies $\rm{TV}(\alpha^*_e,\text{unif}) \leq 2\eps d_{\Pcal} L$.
\end{lemma}
The proof is deferred to the \Cref{app:framework}.
We instantiate \Cref{alg:general} with the following choices. Let $f_t = F(\dots; \lambda_e)$ and $u^*_e = U(\mathcal{L}_e)$ be defined by
\begin{align}\label{eq: sto T^3/4}
f_t(p) = \langle \lambda, u(p,\ell_t)\rangle,\qquad u^*_e = u(p^*(\bar\ell^*_e), \bar\ell^*_e)\,.
\end{align}
\begin{theorem}\label{thm:sto-eff}[(Efficient) Statistically Opportunistic Learner]
Instantiating \Cref{alg:general} with the loss rule and target function defined in \Cref{eq: sto T^3/4} and OGD for both OO oracles yields an $\eps$-statistically opportunistic learner with rate
 $\gamma(T)=\tilde{O}(d_{\Pcal}^\frac{1}{2}T^{-\frac{1}{4}})$, $\alpha(\eps)=(\eps d_{\Pcal})^\frac{1}{3}$.

This algorithm is computationally efficient as long as finding the target $U$ can be efficiently implemented.
\end{theorem}
The proof is again deferred to \Cref{app: thm proofs 3/4}. It is equivalent to the proof of \Cref{thm:strict-eff}, except that it uses \Cref{lem:bound_sto_mean} to obtain an error bound of $\rm{err}\leq O(\eps d_{\Pcal} L)$. 

One simple condition under which \Cref{alg:general} is computationally efficient for the Statistically Opportunistic setting is that the response function is Lipschitz-continuous as demonstrated by the following lemma.
\begin{lemma}\label[lemma]{lem: efficient target}
If the response function satisfies that $g_{\lambda}(\ell) = \langle\lambda, u(p^*(\ell),\ell)\rangle$ is $c$-Lipschitz for any $\bar u$, then setting $u^*_e = u(p^*(\bar \ell_e), \bar\ell_e)$, guarantees that $\gamma(T)=\tilde{O}(d_{\Pcal}^\frac{1}{2}T^{-\frac{1}{4}})$, $\alpha(\eps)=3(\eps d_{\Pcal})^\frac{1}{3}$
\end{lemma}
 
\section{Improved Rates}
\label{sec:improved_rates}
In this section, we show that it is possible to obtain improved rates at the cost of an increase in computational complexity.

Starting from \Cref{cor:general} the main term to improve is $\text{Reg}_{T_e}^{\text{in}}$. This requires the careful design of a new loss rule, since the worst-case regret bound for general linear losses as used in \cref{eq: strict T^3/4} is $\sqrt{T_e}$.
First, we convert the inner OO into an expert problem with arbitrary losses. We assume without loss of generality that $\Pcal$ contains a ball of radius 1 and is contained in a ball of radius $d_{\Pcal}$.
Assume further that the loss functions $f_t$ are 1-Lipschitz, which is for example guaranteed if $f_t(p)$ is piecewise identical to $\langle \lambda_e, u(p,\ell_{t'})\rangle$.
Then there exists a finite cover $\Ncal$ of size $|\Ncal|\leq (3/(d_{\Pcal}^2T))^{d_{\Pcal}}$ such that $\min_{p\in\Ncal}\sum_{t=1}^Tf_t(p)\leq 1+\min_{p\in\Pcal}\sum_{t=1}^Tf_t(p)$ \citep[Lemma 20.1]{lattimore2020bandit}.
This means that we can use any optimization oracle that solves the expert  with loss vectors $y_{t}\in[-1,1]^{|\Ncal|}, y_{t,p}=f_t(p)$, 
even if the loss functions $f_t$ are not convex.
\begin{definition}[Expert Oracle]
An Expert oracle is an algorithm that operates over a finite action set $\mathcal{N}$. It interacts with an environment over $T$ rounds as follows: in each round $t=1, \dots, T$, the oracle first chooses a mixed action $x_t \in \Delta(\Ncal)$, and then receives a loss function vector $y_t \in [-1,1]^{|\Ncal|}$. The oracle provides guarantees on the regret
$$
\mathrm{Reg}_T = \sum_{t=1}^T \langle x_t,y_t\rangle - \min_{x^* \in \Ncal} \sum_{t=1}^T y_{t,x^*}\,.
$$
\end{definition}
Specifically, we make use of the following following oracle.
\begin{lemma}[Expert with path-length bound \citep{chen2021impossible}]
    Given a set of $N$ experts in the expert game with losses bounded in $[-1,1]$. The Multi-scale Multiplicative-weight with Correction (MsMwC) algorithm instantiated with optimistic terms $m_t=\ell_{y-1}$ satisfies
    $$
    \sum_{t=1}^T\langle x_t-e_{i^*}, y_t\rangle \leq O\left(\sqrt{P_T^*\log (NT)} + \log (NT)\right)\,,\text{ where }P_T^*=\sum_{t=1}^{T-1}|y_{t+1,i^*}-y_{t,i^*}|
    \,.$$
\end{lemma}
All what remains is to design loss rules and target functions with small path-length and $\rm{err}$ term.

\subsection{Improved Rates for Strictly Opportunistic Approachability}

Define for $Q_e= \text{conv}(\{\ell_t\}_{t\in\Tcal_e}) $ the solution to the maximin problem
$\ell_e^* := \argmax_{\ell\in Q_e}\min_{p\in\Pcal}\langle \lambda_e, u(p,\ell)\rangle\,.
$
We instantiate \Cref{alg:general} with the following choices. Let $f_t = F(\dots; \lambda_e)$ and $u^*_e = U(\mathcal{L}_e)$ be defined by
\begin{align}\label{eq: strict T^2/3}
f_t(p) = \max_{t'\leq t}\langle \lambda, u(p,\ell_t')\rangle,\qquad u^*_e = u(p^*(\ell_e^*), \ell_e^*)\,.
\end{align}

\begin{theorem}\label{thm:strict-ineff}[Strictly Opportunistic Learner with fast rates]
\Cref{alg:general} with the loss rule and target function defined in \cref{eq: strict T^2/3}, OGD as outer OCO and MsMwC as inner OO with optimal tuning is a strictly opportunistic learner with rate $\gamma(T)=\tilde O(d_{\Pcal}^\frac{1}{3}T^{-\frac{1}{3}})$.
\end{theorem}
The proof is deferred to the \Cref{app: thm proofs 2/3}. It uses that the $\rm{err}$ is non-negative and establishes logarithmic regret for the inner regret due to a constant upper bound on the path length with our choice of $f_t$.

\subsection{Improved Rates for Statistically Opportunistic Approachability}
In the statistical opportunistic case, we have to first adapt the target function to be consistent.
Similarly to the previous section, let $Q_e^{\eps} = Q_{\text{int}}^{\eps L}(\{\ell_t\}_{t\in\Tcal_e})$ and let
\begin{align*}
\ell_e^* := \argmax_{\ell\in Q^\eps_e}\min_{p\in\Pcal}\langle \lambda_e, u(p,\ell)\rangle\,.
\end{align*}
This change is not sufficient to control the error term ($\rm{err}$) since a single outlier at the beginning can increase all successive function values.
We have to operate on a generalized $\max$ function that allows to discard a certain number of outliers.

\begin{definition}
We define the $N$-offset maximum over a set $\mathcal{N}$ of finite cardinality as follows.
\begin{align}
\offmax^{N}_{\mathcal{N}}(f) = \min_{\substack{\mathcal{N}'\subset\mathcal{N}\\|\mathcal{N}'|\geq |\mathcal{N}|-N}}\max_{x\in\mathcal{N}'}f(x)\,,
\end{align}
where by convention, $\offmax^N_{\mathcal{N}}=0$ for $|\mathcal{N}|\leq N$.
\end{definition}

Further define $h_{p,\lambda}(t) = \langle \lambda, u(p,\ell_t)\rangle$. We instantiate \Cref{alg:general} with the following choices. Let $f_t = F(\dots; \lambda_e)$ and $u^*_e = U(\mathcal{L}_e)$ be defined by

\begin{align}\label{eq: sto T^2/3}
f_t(p) = \max\{\offmax^{\eps d T}_{\{t'\in\Tcal_e\,|\,t'\leq t\}}(h_{p,\lambda}), \langle \lambda, u(p,\ell_t)\rangle\},\qquad u^*_e = u(p^*(\ell_e^*), \ell_e^*)\,.
\end{align}
In words, we generally take the $\eps d_{\Pcal} T$-th largest loss we have seen so far for point $p$. Though when the loss of the current round is larger, default to this value to satisfy \Cref{ass:loss_upper_bound}.
\begin{theorem}[Statistically Opportunistic Learner with fast rates]\label{thm:sto-ineff}
\Cref{alg:general} with the loss rule and target function defined in \cref{eq: sto T^2/3}, OGD as outer OCO and MsMwC as inner OO with optimal tuning is a $\eps$-stochastically opportunistic learner with rate $\gamma(T)=\tilde O(d_{\Pcal}^\frac{1}{3}T^{-\frac{1}{3}})$ and $\alpha(\eps)=(\eps d_{\Pcal})^\frac{1}{3}$.
\end{theorem}
The proof is deferred to the \Cref{app: thm proofs 2/3}. Bounding the Hedge regret relies on a careful bound of the Path-Length, while the error term is bounded by $O(\eps d_{\Pcal}L)$ again due to Helly's theorem.
 \section{Discussion on Optimality}
\label{sec:discussion}
The general algorithmic framework presented in \Cref{sec:framework} with instantiations from \Cref{sec:improved_rates} achieves a convergence rate of $O(d_{\Pcal}^{1/3} T^{-1/3})$ for the statistically opportunistic problem. In Lemma~\ref{lem: lower d bound}, we show that the dependence in $d$ is not avoidable. It is natural to ask if the dependence on $d$ and $T$ is optimal? This section explores the question of optimality. We demonstrate that faster $O(T^{-1/2})$ rates are indeed possible in low-dimensional cases of the strict opportunistic approachability problem. However, we also show that these low-dimensional techniques do not readily extend to higher dimensions, leading to a discussion of current bottlenecks.

\subsection{The One-Dimensional Case}

In the simplest setting, where payoffs are one-dimensional ($d=1$), the opportunistic target set $S(Q)$ is simply an interval (the only possible bounded convex set in one dimension). This structure permits a simple yet efficient algorithm that achieves the optimal $O(T^{-1/2})$ rate. The strategy, outlined in Algorithm \ref{alg:1d}, uses a no-regret learner that is reset whenever its average payoff appears to have crossed the target.

\begin{algorithm}[H]
\caption{An algorithm for one-dimensional strictly opportunistic approachability}\label{alg:1d}
\KwIn{Action set $\Pcal$, utility function $u$, horizon $T$}
\BlankLine
Play an arbitrary action $p_1 \in \Pcal$, observe $\ell_1$, and set the target $c \gets u(p^\star(\ell_1), \ell_1)$\;

Initialize a no-regret algorithm \texttt{OCO} over $\Pcal$\;

\For{$t \gets 2$ \KwTo $T$}{
    Determine the side of the target: 
    $s_t \gets \operatorname{sign}\!\left(\frac{1}{t-1}\sum_{i=1}^{t-1} u(p_i,\ell_i) - c\right)$\;
    
    \If{$s_t \neq s_{t-1}$ for some $s_{t-1}$ in the history}{
        Reset \texttt{OCO}\;
    }
    
    Play action $p_t$ from \texttt{OCO}\;
    
    Observe $\ell_t$ and feed the signed loss 
    $f_t(\cdot) \gets s_t\, u(\cdot,\ell_t)$ to \texttt{OCO}\;
}
\end{algorithm}

The intuition behind this algorithm is that the sign $s_t$ remains constant for long periods. If it were to flip frequently, the average payoff would oscillate around $c$. A no-regret learner, when fed losses that consistently push it in one direction (e.g., losses are always positive), will quickly learn to counteract them, preventing such oscillation. Consequently, any resets must be infrequent. The analysis of the long final epoch after the last reset, where the sign $s_N$ is fixed, yields the desired $O(\sqrt{\log|\Pcal|/T})$ rate on the distance to the target set $S(Q)$. 
\begin{theorem}\label{thm:one-dim}
For an instance with payoff dimension $d=1$ and a strictly $Q$-restricted adversary, Algorithm \ref{alg:1d} with target $c \in S(Q)$ ensures $\dist(\bar{u}_T, S(Q)) \le O\left(\sqrt{\log|\Pcal|/T}\right)$.
\end{theorem}
The proof of this theorem is deferred to \Cref{app:discussion}.

\subsection{Fast Rates for Strictly Opportunistic Approachability in Low Dimensions}

For dimensions $d > 1$, the primary challenge is that the learner must discover the geometry of the unknown adversary set $Q = \mathrm{conv}(\{\ell_t\}_{t=1}^T)$ over time. A natural approach for the strictly opportunistic case is to maintain an estimate of this set, $Q_t = \mathrm{conv}(\{\ell_1, \dots, \ell_{t-1}\})$, and play as if $S(Q_t)$ were the true target. The total error then accumulates based on how much the set expands when a new point $\ell_t$ arrives, an error proportional to $\dist(\ell_t, Q_t)$. Remarkably, in low dimensions, the sum of these distances is bounded. The following lemma establishes this for $d=2$.

\begin{lemma}\label[lemma]{lem:incremental-convex-hull}
Let $\{x_t\}_{t=1}^T$ be a sequence of points in $\mathbb{R}^2$ with $\|x_t\|_2 \le 1$. Let $K_t = \mathrm{conv}(\{x_1, \dots, x_{t-1}\})$ for $t \ge 2$. Then the cumulative distance is bounded: $\sum_{t=2}^{T} \dist(x_t, K_t) = O(\sqrt{T})$.
\end{lemma}
The proof of this lemma is deferred to \Cref{app:discussion}. By running an approachability algorithm against the expanding target set $S(Q_t)$ and absorbing the error from the projections, we can prove the following theorem.

\begin{theorem}\label{thm:two-dim}
For strictly opportunistic approachability with a two-dimensional
\footnote{We believe this geometric argument, based on the trade-off between volume and surface area, can be extended to $d=3$, also yielding an $O(T^{-1/2})$ rate.} action space $\Lcal$, there exists an efficient algorithm with a convergence rate of $O(1/\sqrt{T})$.
\end{theorem}
The proof of this lemma is deferred to \Cref{app:discussion}.
The geometric approach breaks down for $d > 3$. A straightforward extension of Lemma~\ref{lem:incremental-convex-hull} to $d$ dimensions bounds the sum of the $d$-th powers of the distances, which only gives a cumulative distance of $O(T^{1-1/d})$. This rate is worse than $O(\sqrt{T})$ when $d > 3$.

 \newpage

\clearpage
\bibliography{ref.bib}
\clearpage
\appendix
\section{Deferred Proofs from \Cref{sec:framework}}
\label{app:framework}

\begin{proofof}[\Cref{cor:general}]
The outer OCO operates over the action set $\Bcal_d(1)$, so $D=2$. The norm of the gradient $\|g_e\|$ is bounded by twice the maximum norm of payoffs, which implies $G=2$. Applying Lemma~\ref{lem:oco_guarantee} to $\text{Reg}^{\text{in}}_{L}$ finishes the proof.
\end{proofof}

\begin{proofof}[\Cref{lem:bound_sto_mean}]
Define $Q_{\text{close}}=\text{conv}(\{\frac{1}{|\Tcal'|}\sum_{t\in\Tcal'}\ell_t\,|\,\Tcal'\subset \Tcal_e, |\Tcal'|\geq T_e-\eps d_{\Pcal} T\})$.
If we can show that $Q_{\text{close}}\cap Q^\eps_e\neq \emptyset$, then this implies the Lemma since for every $\Tcal'\subset\Tcal_e, |\Tcal'|\geq T_e-\eps d_{\Pcal} T$ we have
\begin{align*}
\sum_{t\in \Tcal'}(\frac{1}{|\Tcal|'}-\frac{1}{T_e}) + \sum_{t\in \Tcal_e\setminus \Tcal'}\frac{1}{T_e})\leq 1-\frac{T_e-\eps d_{\Pcal} T}{T_e}+\frac{\eps d_{\Pcal} T}{T_e}= 2\eps d_{\Pcal} L\,,
\end{align*}
by using $T= L T_e$.
By convexity of $\rm{TV}$, the total variation to the uniform sample is bounded for every element in $Q_{\text{close}}$.

Finally we show that the intersection is non-empty.
Assume the intersection is empty and let $H$ be an open half-space containing $Q_{\text{close}}$ but not $Q^\eps_e$. By \Cref{lem:max-points-outside-Q}, the set $H$ contains no more than $\eps d T$ points. Take $\bar\ell$ as the mean of all points not in $H$. By definition, $\bar\ell\in Q_{\text{close}}$. Hence such a half-space cannot exist, which implies that
$Q_{\text{close}}\cap Q^\eps_e\neq \emptyset$.
\end{proofof}

\begin{proofof}[\Cref{lem: efficient target}]
While $u^*_e$ is not a consistent target function it is sufficiently close to one as demonstrated by the following
    \begin{align*}
        &\max_{\lambda \in B_d(1)}\min_{\ell' \in Q_{\text{int}}^\eps}\langle \bar u, u(p^*(\bar \ell^*_e), \bar\ell^*_e) - u(p^*(\ell'),\ell') \rangle\\
        &\max_{\lambda \in B_d(1)}\langle \bar u, u(p^*(\bar \ell^*), \bar\ell^*) - u(p^*(\bar\ell_e),\bar\ell_e) \rangle\\
        &\leq c\|\bar\ell^*_e-\bar\ell_e\|\\
        &\leq c\rm{TV}(\alpha^*_e,\text{unif})\max_{t}\|\ell_t\| \leq 2c\eps d_{\Pcal}L
    \end{align*}
\end{proofof}

\subsection{Proof of Theorems}
\label{app: thm proofs 3/4}
\begin{proofof}[\Cref{thm:strict-eff}]
Since $\bar\ell_e\in Q_{\text{int}}^0$, this target function is consistent with $S_{\text{int}}^0$ and the loss rule satisfies \Cref{ass:loss_upper_bound}.
We additionally have $\rm{err}\leq 0$, as we see by taking the definition
\begin{align*}
\min_{p\in\mathcal{P}}\left(\sum_{t\in\mathcal{T}_e}f_t(p)-\langle\lambda_e, u^*_e\rangle\right) &=
\min_{p\in\mathcal{P}}\left(\sum_{t\in\mathcal{T}_e}\langle\lambda_e, u(p, \ell_t)-u(p^*(\bar\ell_e),\bar\ell_e))\rangle\right)\\
&=T_e\min_{p\in\mathcal{P}}\left(\langle\lambda_e, u(p, \bar\ell_e)-u(p^*(\bar\ell_e),\bar\ell_e))\rangle\right)\leq 0\,.
\end{align*}
Applying Corollary~\ref{cor:general} yields
\begin{align*}
    \dist(\bar u_T,S_{\text{int}}) \leq \frac{1}{T_e}\text{Reg}_{T_e}^{\text{in}} + O(L^{-\frac{1}{2}})\,.
\end{align*}
Using OGD as the inner OO, we have $\text{Reg}_{T_e}^{\text{in}}\leq \tilde{O}(d_{\Pcal}\sqrt{ T_e})$. Tuning $L= \sqrt{T}/d_{\Pcal}$ completes the proof.
\end{proofof}

\begin{proofof}[\Cref{thm:sto-eff}]
    We obtain the following bound on the error term.
    \begin{align*}
        \frac{1}{T_e} \min_{p \in \Pcal} \sum_{t \in \mathcal{T}_e} f_t(p) - \langle \lambda_e, u^*_e \rangle&\leq 
        \frac{1}{T_e} \sum_{t \in \mathcal{T}_e} \langle \lambda_e, u(p^*(\bar\ell^*_e), \ell_t)  - \langle \lambda_e, u(p^*(\bar\ell^*_e),\bar\ell^*_e) \rangle\\
        &=\sum_{t\in\Tcal_e}(T_e^{-1}-\alpha^*_e(t))\langle \lambda_e, u(p^*(\bar\ell^*_e),\ell_t \rangle
        \\&\leq \rm{TV}(\alpha^*_e,\text{unif})
        \leq 2\eps d_{\Pcal} L\tag{\Cref{lem:bound_sto_mean}}\,.
    \end{align*}
    Plugging in the OGD bound of \Cref{lem:oco_guarantee} and tuning $L= O(\min\{\sqrt{T}/d_{\Pcal}, (\eps d_{\Pcal})^{-\frac{2}{3}} \}$ completes the proof.
\end{proofof}
\section{Deferred Proofs from \Cref{sec:improved_rates}}
We require the following technical Lemmas.

\begin{lemma}
\label[lemma]{lem: offmax properties}
    The offmax satisfies for any $\mathcal{M}\subset\mathcal{N}$: $\offmax_{\mathcal{M}}^N(f) \geq \offmax_{\mathcal{N}}^{N+|\mathcal{N}|-|\mathcal{M}|}(f)$ as well as $\sum_{s=0}^{\mathcal{N}}\offmax_{\mathcal{N}}^s(f)=\sum_{x\in\mathcal{N}}f(x)$.
\end{lemma}
\begin{proof}
    For the first statement, let $\mathcal{N}'$ be the argmin subset in the definition of $\offmax_{\mathcal{M}}^N$. $\mathcal{N}'$ is also a valid subset in the minimization for $\offmax_{\mathcal{N}}^{N+|\mathcal{N}|-|\mathcal{M}|}$, hence the value cannot be larger.

    For the second, note that the offset max is iterating through all elements ordered by size, hence we are simply summing over the set.
\end{proof}

We have the following geometric properties. We first recall the definition of an intersection set
\begin{definition}
Let $\mathcal{N}\subset \mathbb{R}^d$ with $|\mathcal{N}|=T$, and fix $\varepsilon<1/(d+1)$. Define
\[
\mathbf Q_\varepsilon := \Bigl\{\operatorname{conv}(\mathcal{N}\setminus R): R\subseteq \mathcal{N},\, |R|\le \varepsilon T\Bigr\},\qquad
Q_{\mathrm{int}} := \bigcap_{Q\in\mathbf Q_\varepsilon} Q.
\]
\end{definition}

\begin{lemma}\label[lemma]{lem:max-points-outside-Q}
Let $\mathcal{N}\subset \mathbb{R}^d$ with $|\mathcal{N}|=T$, and fix $\varepsilon<1/(d+1)$. Let $H\subset \mathbb{R}^d$ be an open halfspace with $H \cap Q_{\mathrm{int}}^\eps = \varnothing$. Then
\[
|H \cap \mathcal{N}| \le d\,\varepsilon T.
\]
\end{lemma}
\begin{proof}
Let $\mathcal H_\varepsilon$ denote the family of open halfspaces each containing at most $\varepsilon T$ points of $\mathcal{N}$. By construction of $Q_{\mathrm{int}}$, for any $x\notin Q_{\mathrm{int}}$ there exists $H_x\in\mathcal H_\varepsilon$ with $x\in H_x$. In particular, $H \subseteq \bigcup_{H_x\in \mathcal H_\varepsilon} H_x$.

Consider the family $\mathcal H_\varepsilon \cup \{H^c\}$, which covers $\mathbb{R}^d$ (since $H^c$ covers $Q_{\mathrm{int}}$ and the $H_x$ cover $H$). By the Helly–Danzer–Grünbaum–Klee theorem (\cite[Thm.~2.3]{danzer1963helly}), there exists a subfamily of at most $d+1$ halfspaces that still covers $\mathbb{R}^d$. One of these must be $H^c$ (the only halfspace covering $Q_{\mathrm{int}}$), leaving at most $d$ halfspaces from $\mathcal H_\varepsilon$ that cover $H$. Each of these contains at most $\varepsilon T$ points, so
\[
|H \cap \mathcal{N}| \le d \cdot \varepsilon T,
\]
\end{proof}

\begin{lemma}\label[lemma]{lem: offmax versus minimax}
    Let $B:\mathcal{A}\times\mathcal{L}\rightarrow \mathcal{R}$ be a bi-affine function. Let $\mathcal{N}$ be a set of points in $\mathcal{L}$ and $Q_{\text{int}}$ be the intersection set of N points removed. It holds
    \begin{align*}
    \min_{p\in\mathcal{P}}\max_{\ell\in Q_{\text{int}}}B(p,\ell)\geq
    \min_{p\in\mathcal{P}}\offmax^{dN}_{\mathcal{N}}(B(p,\cdot))\,.
    \end{align*}
\end{lemma}
\begin{proof}
    The inequality follows from Lemma~\ref{lem:max-points-outside-Q}. Define the half-space $H_p = \{\ell\in \mathbb{R}^d\,|\,B(p,\ell)>\max_{\ell'\in Q_{\text{int}}}B(p,\ell')\}$. This is a half-space due to $B$ being bi-affine. $H$ contains no more than $dN$ points by Lemma~\ref{lem:max-points-outside-Q}. Hence taking the set $\mathcal{N}'$ that removes all points in $H_p$ (and additional points at random to fill up the $dN$ points) satisfies $\max_{\ell\in\mathcal{N}'}B(p,\ell)\leq \max_{\ell\in Q_{\text{int}}}B(p,\ell)$ This holds for any $p$, specifically for the argmin of the LHS.
\end{proof}

\subsection{Main proofs}
\label{app: thm proofs 2/3}

\begin{proofof}[\Cref{thm:sto-ineff}]
We first ensure that $U$ is consistent. In the construction of $Q_e^{\eps}=Q_{\text{int}}^{\eps T}(\Tcal_e)$, we are removing an $\eps L$ fraction of points from $\Tcal$. Since $|\Tcal|=T/L$, we remove a total of $\eps T$ points, which is the same as in $Q_{\text{int}}^{\eps}$. Since $\Tcal_e\subset (\ell_t)_{t\in[T]}$, we have $Q_{e}^{\eps}\subset Q_{\text{int}}^\eps$, which ensures consistency of $U$.

Applying Corollary~\ref{cor:general}, with the regret bound of MsMwC, we need to bound the path length and the error term.

\textbf{Path Length.} Recall, given $h_p(t)=\langle \lambda_e, u(p,\ell_t)\rangle$, the path length is defined by
\begin{align*}
    &P^\star = \sum_{t=(e-1)T_e+1}^{eT_e-1}|f_{t+1}(p)-f_t(p)|\\
    &f_t(p) = \max\{\offmax^{\eps d T}_{\Tcal_e}(h_p), h_p(t)\}\,.
\end{align*}
We bound the Path length of this loss as follows. Let $\Tcal_e^t=\{\ell_{t'}\,|\,t'\in[(e-1)T_e+1,t]\}$:
\begin{align*}
    &|f_{t+1}(p)-f_t(p)| 
    \\&= |\max\{h_p(t+1),\offmax^{\eps d T}_{\Tcal_e^t}(h_p)\}-\max\{h_p(t),\offmax^{\eps d T}_{\Tcal_e^{t-1}}(h_p)\}|
    \\&\leq\max\{h_p(t+1)-\offmax^{\eps d T}_{\Tcal_e^{t-1}}(h_p),h_p(t)-\offmax^{\eps d T}_{\Tcal_e^t}(h_p),\offmax^{\eps d T}_{\Tcal_e^t}(h_p)-\offmax^{\eps d T}_{\Tcal_e^{t-1}}(h_p)\}\\
    &\leq (h_p(t+1)-\offmax^{\eps d T}_{\Tcal_e^{t-1}}(h_p))_++(h_p(t)-\offmax^{\eps d T}_{\Tcal_e^t}(h_p))_++\offmax^{\eps d T}_{\Tcal_e^t}(h_p)-\offmax^{\eps d T}_{\Tcal_e^{t-1}}(h_p)
\end{align*}
Summing over $t$ and using $(h_p(t)-\offmax^{\eps d T}_{\Tcal_e^t}(h_p))_+\leq (h_p(t)-\offmax^{\eps d T}_{\Tcal_e^{t-2}}(h_p))_+$ yields
\begin{align*}
    \sum_{t=(e-1)T_e+1}^{eT_e-1}|f_{t+1}(p)-f_t(p)| \leq 2\sum_{t=1}^T (h_p(t)-\offmax^{\eps d T}_{\Tcal_e^{t-2}}(h_p))_+ +\underbrace{\offmax^{\eps d T}_{\Tcal_e}(h)}_{\leq 1}
\end{align*}
To bound the remaining term, notice that we can assume wlog that $(h_p(t)-\offmax^{\eps d T}_{\Tcal_e^{t-2}}(h_p))\geq 0$. Otherwise take the subset of $\Tcal_e$ for which this holds and repeat the argument below on this subset. Then
\begin{align*}
\sum_{t=(e-1)T_e}^{eT_e} (h_p(t)-\offmax^{\eps d T}_{\Tcal_e^{t-2}}(h_p))&\leq \sum_{t=(e-1)T_e}^{eT_e}\left(f_a(i_s)-\offmax^{\eps d T}_{\Tcal_e^{t-2}}(h_p)\right)\tag{by Lemma~\ref{lem: offmax properties}}
    \\&\leq \sum_{s=1}^{T_e}\left(\offmax_{\Tcal_e}^{s-1}(h_p)-\offmax^{\eps d T+|\Tcal_e|+2-s}_{\Tcal_e}(h_p)\right)
    \\&=\sum_{s'=0}^{\eps d T +1}\offmax_{\Tcal_e}^{s'}(h_p)\leq \eps d T +2
\end{align*}
Hence the path length is bounded by $O(\eps d_{\Pcal} T)$. By AM-GM inequality, we have $\text{Reg}^{\text{in}}_{T_e} = O(\eps d_{\Pcal} T+\log|\Ncal| )$.

\textbf{Error term.}
\begin{align*}
    \rm{err}&=\max_e\min_{p\in\Pcal}\frac{1}{T_e}\sum_{t\in\Tcal_e}\left(f_t(p)-\langle \lambda_e, u^\star_e\rangle\right)\\
    &=\max_e\min_{p\in\Pcal}\frac{1}{T_e}\sum_{t\in\Tcal_e}\left(\max\{\offmax^{\eps d_{\Pcal}T}_{\Tcal_e^t}(h_p), \langle \lambda_e, u(p,\ell_t)\rangle\}-\langle \lambda_e, u(p^\star(\bar\ell_e^\eps),\bar\ell_e^\eps)\rangle\right)\\
    &\leq \max_e\frac{1}{T_e}\sum_{t\in\Tcal_e}\left(\max\{\offmax^{\eps d_{\Pcal}T}_{\Tcal_e}(h_{p^\star(\bar\ell_e^\eps)}), \langle \lambda_e, u(p^\star(\bar\ell_e^\eps),\ell_t)\rangle\}-\langle \lambda_e, u(p^\star(\bar\ell_e^\eps),\bar\ell_e^\eps)\rangle\right)\\  
    &\leq \max_e\offmax^{\eps d_{\Pcal}T}_{\Tcal_e}(h_{p^\star(\bar\ell_e^\eps)})-\langle \lambda_e, u(p^\star(\bar\ell_e^\eps),\bar\ell_e^\eps)\rangle
+\frac{1}{T_e}\sum_{t\in\Tcal_e}\mathbb{I}\left(h_{p^\star(\bar\ell_e^\eps)} > \offmax^{\eps d_{\Pcal}T}_{\Tcal_t}(h_{p^\star(\bar\ell_e^\eps)})\right)\\
&\leq \frac{1}{T_e}\sum_{t\in\Tcal_e}\mathbb{I}\left(h_{p^\star(\bar\ell_e^\eps)} > \offmax^{\eps d_{\Pcal}T}_{\Tcal_t}(h_{p^\star(\bar\ell_e^\eps)})\right)\tag{by Lemma~\ref{lem: offmax versus minimax}}\\
&\leq \frac{\eps d_{\Pcal}T}{T_e} = \eps d_{\Pcal} L\,.
\end{align*}
Applying \Cref{cor:general} with the bound on the regret and error term, we have
\begin{align*}
\dist(\bar u_T, S_{\text{int}}^\eps) \leq \tilde O\left( \frac{d_{\Pcal}L}{T}+\frac{1}{\sqrt{L}} + \eps d_{\Pcal} L\right)\,.
\end{align*}
Setting $L=\min\left\{(T/d_{\Pcal})^\frac{2}{3}, (\eps d_{\Pcal})^{-\frac{2}{3}}\right\}$ completes the proof.
\end{proofof}

\begin{proofof}[\Cref{thm:strict-ineff}]
The choice of loss rule satisfies \Cref{ass:loss_upper_bound} and the target function is consistent with $S_{\text{int}}^0$, hence we can apply \Cref{cor:general} and only need to bound $\text{Reg}^{\text{In}}_{T_e}$ and $\rm{err}$.
For the error, we have
\begin{align*}
    \min_{p\in\Pcal}\sum_{t\in\mathcal{T}_e}f_t(p) &\leq T_e\min_{p\in\Pcal}\max_{t\in\Tcal_e}\langle\lambda_e,u(p,\ell_{t'})\rangle
    \\&=T_e\max_{\ell\in Q_e}\min_{p\in\Pcal}\langle\lambda_e,u(p,\ell)\rangle
    \\&=T_e\min_{p\in\Pcal}\langle\lambda_e,u(p,\ell^*_e)\rangle\\
    &\leq T_e\langle\lambda_e,u(p^*(\ell^*_e),\ell^*_e)\rangle \,,
\end{align*}
hence $\rm{err}\leq 0$.
For the inner optimization, we need to bound the path length. By monotonicity of the losses, we have $f_{t+1} \geq f_t$, hence for any $p\in\Ncal$: $\sum_{t=(e-1)T_e +1}^{eT_e-1}|f_{t+1}(p)-f_{t}(p)|=\sum_{t=(e-1)T_e +1}^{eT_e-1}(f_{t+1}(p)-f_{t}(p))=f_{eT_e}(p)-f_{(e-1)T_e+1}(p)\leq 2$.
The inner regret is $\text{Reg}^{\text{in}}_{T_e}=O(\log|\Ncal|)$.
Plugging these in, we have
\begin{align*}
    \dist(\bar u_T, S_t)\leq \frac{1}{T_e}\log(|\Ncal|) + \sqrt{1/L}\,.
\end{align*}
Using a discretization of size $\log|\Ncal| = O(d_{\Pcal}\log(T))$ and tuning $L=(T/d_{\Pcal})^\frac{2}{3}$ completes the proof.
\end{proofof}

\section{Deferred Proofs from \Cref{sec:discussion}}
\label{app:discussion}
\begin{proofof}[\Cref{thm:one-dim}]
Let $N$ be the time of the last reset, making $s_t=s_N$ for all $t \ge N$. Let $c^* \in S(Q)$ be the point closest to $\bar{u}_T$.

The no-regret guarantee of the \texttt{OCO} instance over the final epoch $\{N, \dots, T\}$ of length $T' = T-N+1$ implies:
\begin{align*}
    s_N \sum_{t=N}^T (u(p_t, \ell_t) - c^*) \le \text{Reg}_{T'} = O(\sqrt{\log|\Pcal|T'})
\end{align*}
We use this to bound one side of the total deviation. Since resets are infrequent for a no-regret learner, we can treat $N=O(1)$, which means the initial sum $\sum_{t=1}^{N-1}(u_t-c^*)$ is also $O(1)$ and $\text{Reg}_{T'}=O(\sqrt{\log|\Pcal|T})$.
\begin{align*}
    s_N T (\bar{u}_T - c^*) &= s_N \left(\sum_{t=1}^{N-1} (u_t - c^*) + \sum_{t=N}^T (u_t - c^*) \right) \\
    &\le O(1) + \text{Reg}_{T'} = O(\sqrt{\log|\Pcal|T})
\end{align*}
This gives the upper bound on the signed distance: $s_N (\bar{u}_T - c^*) \le O\left(\sqrt{\log|\Pcal|/T}\right)$.

For the lower bound, the no-reset condition at $t=T$ is $s_N(\bar{u}_{T-1} - c) \ge 0$.
\begin{align*}
    s_N(\bar{u}_T - c) &= s_N\left(\frac{T-1}{T}\bar{u}_{T-1} + \frac{u_T}{T} - c \right) \\
    &= \frac{T-1}{T}s_N(\bar{u}_{T-1}-c) + \frac{s_N(u_T-c)}{T} \ge \frac{s_N(u_T-c)}{T} \ge -O(1/T)
\end{align*}
Assuming the chosen target $c$ is representative of the true target set $S(Q)$, the upper and lower bounds together imply that the absolute distance $|\bar{u}_T - c^*|$ is small. The two bounds combine to show:
\[
\dist(\bar{u}_T, S(Q)) = |\bar{u}_T - c^*| \le O\left(\sqrt{\frac{\log|\Pcal|}{T}}\right)
\]
\end{proofof}

\begin{proofof}[Lemma~\ref{lem:incremental-convex-hull}]
We will show that $\sum_{t=2}^{T} d(x_t, K_t)^2 = O(1)$. It then follows by Cauchy-Schwartz that $\sum_{t=2}^{T} d(x_t, K_t) = O(\sqrt{T})$.

For any $t \geq 2$, let $y_t$ be the Euclidean projection of $x_t$ onto $K_t$. Let $h$ be the line passing through $y_t$ perpendicular to $x_t$. The line $h$ intersects the next set $K_{t+1} = \conv(K_t \cup \{x_t\})$ in some line segment $S$ of width $w$. We consider two cases:

\begin{itemize}
    \item \textbf{Case 1: $w \geq d(x_t, K_t)$.} In this case, we will argue that the area of $K_{t}$ increases by at least $\Omega(d(x_t, K_t)^2)$ upon adding the point $x_{t}$. In particular, note that $K_{t+1} \setminus K_{t}$ contains the triangle with base $S$ and apex $x_{t+1}$, which has width $w$ and height $d(x_t, K_t)$. The area of this triangle is therefore at least $\Omega(w \cdot d(x_t, K_t)) = \Omega(d(x_t, K_t)^2)$.
    \item \textbf{Case 2: $w < d(x_t, K_t)$.} In this case, we will argue that the perimeter of $K_t$ increases by at least $\Omega(d(x_t, K_t))$ upon adding the point $x_t$. In particular, let $K' = \conv(K_t \cup S)$; note that $K_t \subseteq K' \subseteq K_{t+1}$, and so $\mathrm{Perim}(K_t) \leq \mathrm{Perim}(K') \leq \mathrm{Perim}(K_{t+1})$. But now, $\mathrm{Perim}(K_{t+1}) - \mathrm{Perim}(K') \geq 2\cdot d(x_t, K_t) - w$ (since it gains the two legs of the previously described triangle, each of which have length at least $d(x_t, K_T)$ but loses the base $S$).
\end{itemize}

Since both the area and perimeter of any convex subset of the unit disk is upper bounded by a constant, this implies that $\sum_{t=2}^{T} d(x_t, K_t)^2$ is upper bounded by a constant, as desired.
\end{proofof}

\begin{proofof}[\Cref{thm:two-dim}]
The algorithm maintains an evolving estimate of the adversary's action set, \textbf{$Q_t = \mathrm{conv}(\{\ell_1, \dots, \ell_{t-1}\})$}. At each step $t$, it runs a standard approachability algorithm, feeding it the projection of the true action, $\hat{\ell}_t = \Pi_{Q_t}(\ell_t)$, instead of $\ell_t$ itself.

The total error of the true average payoff $\bar{u}_T$ is decomposed using the triangle inequality, where $\hat{u}_T$ is the average payoff from the projected game:
\[
\dist(\bar{u}_T, S(Q_{T+1})) \le \underbrace{\|\bar{u}_T - \hat{u}_T\|}_{\text{Projection Error}} + \underbrace{\dist(\hat{u}_T, S(Q_{T+1}))}_{\text{Approachability Error}}
\]

The projection error is bounded by the cumulative distance of the projections. By the bilinearity property of the payoff function and \Cref{lem:incremental-convex-hull}:
    \[
    \|\bar{u}_T - \hat{u}_T\| \le \frac{1}{T} \sum_{t=1}^T \dist(\ell_t, Q_t) = \frac{1}{T} O(\sqrt{T}) = O\left(\frac{1}{\sqrt{T}}\right).
    \]

The approachability error follows the standard rate of $O(1/\sqrt{T})$.
\end{proofof}
 
\end{document}